\documentclass{article}
\usepackage{ijcai25}
\usepackage{xcolor}

\usepackage{tikz}
\usepackage{pgfplots}
\usepgfplotslibrary{fillbetween}
\pgfplotsset{compat=1.18}

\usepackage{times}
\usepackage{soul}
\usepackage{url}
\usepackage[hidelinks]{hyperref}
\usepackage[utf8]{inputenc}
\usepackage[small]{caption}
\usepackage{graphicx}
\usepackage{amsmath}
\usepackage{amsthm}
\usepackage{booktabs}
\usepackage{algorithm}
\usepackage{algpseudocode}
\usepackage[switch]{lineno}
\usepackage{amsmath, amssymb, amsthm}
\usepackage{cleveref}

\newcommand{\indicator}[1]{\mathbf{1}\left\{#1\right\}}

\newtheorem{theorem}{Theorem}[section]
\newtheorem{lemma}[theorem]{Lemma}
\newtheorem{definition}[theorem]{Definition}
\newtheorem{proposition}[theorem]{Proposition}

\title{Probably Approximately Consensus: \\ 
On the Learning Theory of Finding Common Ground}

\author{
Carter Blair$^{1\dagger}$
\and
Ben Armstrong$^2$\and
Shiri Alouf-Heffetz$^3$\and
Nimrod Talmon$^3$\and
Davide Grossi$^4$
\affiliations
$^1$University of Waterloo, $^2$Tulane University, $^3$Ben Gurion University, $^4$University of Groningen\\
\emails
$^\dagger$carter.blair@uwaterloo.ca
}

\begin{document}

\maketitle

\begin{abstract}
A primary goal of online deliberation platforms is to identify ideas that are broadly agreeable to a community of users through their expressed preferences. Yet, consensus elicitation should ideally extend beyond the specific statements provided by users and should incorporate the relative salience of particular topics. We address this issue by modelling consensus as an interval in a one-dimensional opinion space derived from potentially high-dimensional data via embedding and dimensionality reduction. We define an objective that maximizes expected agreement within a hypothesis interval where the expectation is over an underlying distribution of issues, implicitly taking into account their salience. We propose an efficient Empirical Risk Minimization (ERM) algorithm and establish PAC-learning guarantees. Our initial experiments demonstrate the performance of our algorithm and examine more efficient approaches to identifying optimal consensus regions. We find that through selectively querying users on an existing sample of statements, we can reduce the number of queries needed to a practical number.
\end{abstract}

% \section{Main Question}

% \begin{itemize}
%     \item Can we find a region in a one-dimensional opinion space that captures some area of maximal approval?
% \end{itemize}

% \begin{enumerate}
%     \item Passive PAC: Given a sample of points from a distribution $\mathcal{D}_I$, how do we best identify the region of maximum consensus?
%     \item Active PAC: Given voters each having a single region of consensus, how do we query the voters to best identity each voter's approval region?
% \end{enumerate}

\section{Introduction}
\label{sec:introduction}

Identifying regions of maximal approval, or common ground, within an opinion space is a fundamental task for understanding collective sentiment and informing group decisions. This is particularly relevant for text-based online deliberation platforms such as Polis\footnote{\url{https://pol.is}} \cite{small2021polis} and Remesh\footnote{\url{https://www.remesh.ai/}}, which aim to distill areas of agreement from complex social discussions. In some approaches, like those used by Polis, opinions are represented as points in a metric space derived through embedding and dimensionality reduction where proximity signifies similarity. This spatial approach to opinion representation has roots in both political theory \cite{merrill1999unified} and social choice \cite{bulteau2021aggregation} and allows for a geometric interpretation of consensus.

Platforms for online deliberation typically gather opinions of participants through their approval or disapproval of specific statements. While a statement represents a single point in the broader opinion space a set of broadly agreeable statements  
% may extend beyond a single point to 
can be naturally viewed as a contiguous region of related viewpoints. Separately, the \textit{salience} of the issues is critical for meaningful consensus. For instance, in a discussion about AI development, a statement like ``Historically most progress on AI has been made in academia'' might receive universal approval. However, this may be an opinion that all users take for granted, and thus may not be at the core of the debate. In contrast, opinions on whether AI models should be open-source, open-weights, or closed-source are discussed and debated more broadly. If an opinion in this ``intellectual property'' region is also highly approved, it might represent a more meaningful consensus statement. Thus, an ideal consensus region should reflect high approval on salient segments of the discussion, rather than primarily including points receiving widespread approval, which are less relevant.

Our goal is to query participants about their approval preferences in order to find a region of responses that maximizes approval. In this paper, we assume that preferences are single-peaked along one dimension. This one-dimensional (1D) space can be thought of as a relevant dimension derived from higher-dimensional opinion space (e.g., via embedding and dimensionality reduction). We focus on the \textit{passive learning} setting: given a sample of issues (points) drawn from an underlying distribution and labelled by voters, how do we best identify the interval representing maximum consensus? We formalize this problem, provide an efficient algorithm for finding an empirical solution, and derive PAC-style theoretical guarantees on its performance. This work serves as a foundation for more complex scenarios, including active learning where one might query voters strategically.\footnote{In the Polis context, this is referred to as {\em opinion} or {\em comment routing} \cite{small2021polis}.}

The main contributions of this paper are:
\begin{itemize}
    \item A formal problem definition for passive 1D interval-based consensus finding, incorporating the notion of issue salience through an underlying distribution.
    \item An efficient Empirical Risk Minimization (ERM) algorithm for this problem.
    \item PAC learning analysis, including the pseudo-dimension of the function class and sample complexity bounds.
    \item Experimental validation and exploration of query strategies that make learning optimal regions dramatically more efficient.
\end{itemize}

\subsection{A Concrete Example}
\label{sec:concrete_example}

To illustrate our one-dimensional model, consider the ongoing discussion surrounding artificial intelligence (AI) safety. This debate encompasses a spectrum of views, from strong advocacy for rapid, unhindered AI development to calls for stringent regulation or even moratoria due to safety concerns\footnote{https://futureoflife.org/open-letter/pause-giant-ai-experiments/}. Such opinions, though multifaceted, can often be projected onto a principal axis representing, for example, a ``progress vs. precaution'' spectrum. Along this dimension, individuals may have an interval of acceptable stances, as depicted by the green regions in \autoref{fig:opinion-agreeability}. Further, the distribution of opinions (depicted in yellow) might have more mass around moderate positions as a larger number of moderate opinions were put forward by users. 

\begin{figure}[t]
    \centering
    \includegraphics[keepaspectratio,width=0.9\columnwidth]{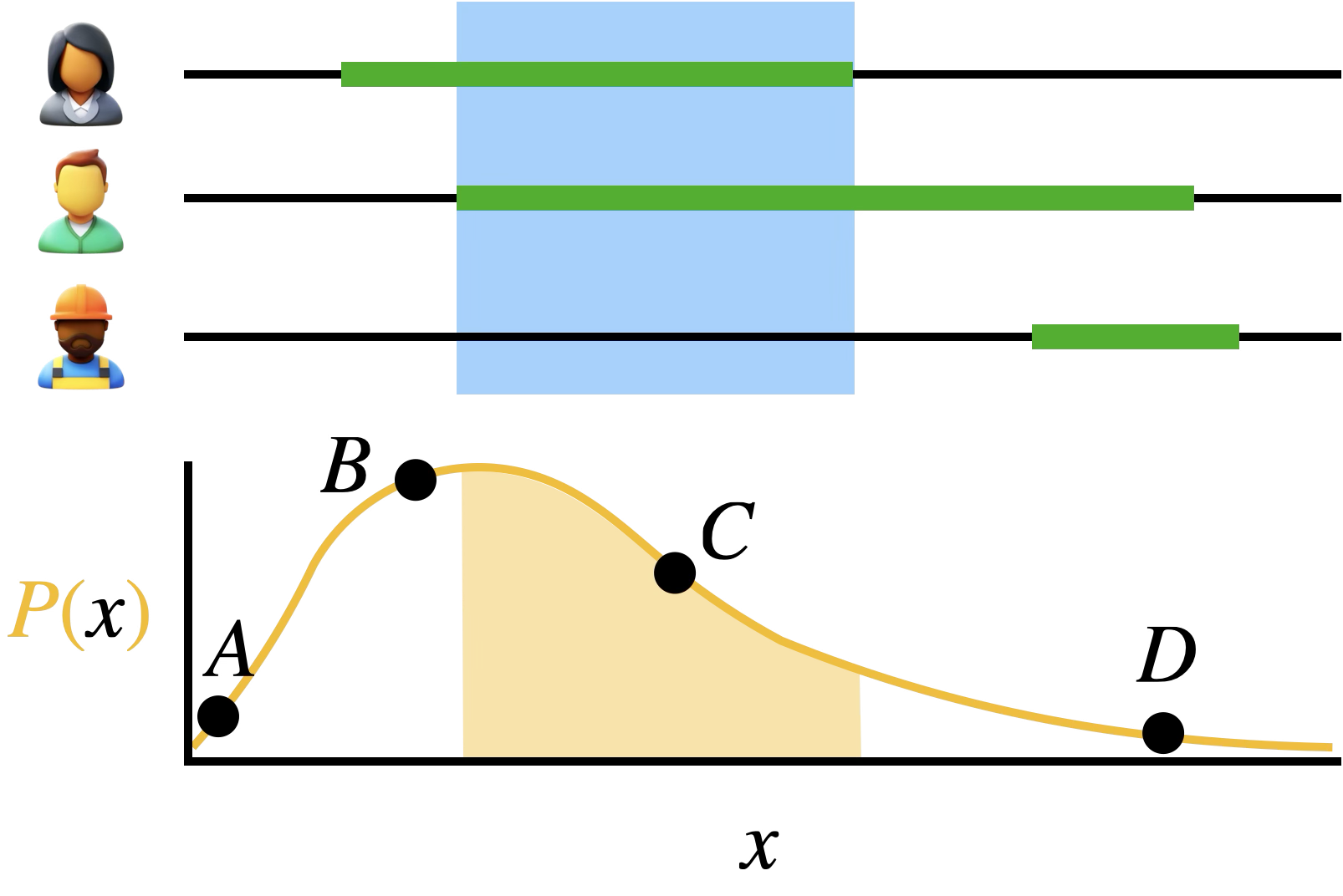}
    \caption{A conceptual representation of approval intervals for different users along a 1D opinion spectrum.}
    \label{fig:opinion-agreeability}
\end{figure}

% \begin{figure}[t]
%     \centering
%     \includegraphics[keepaspectratio,width=\columnwidth]{plots/opinion_probabilities.png}
%     \caption{A hypothetical probability distribution \(P(x)\) of opinions along the AI safety spectrum. \ben{Needs updating. Shouldn't match the agreement intervals above.}}
%     \label{fig:opinion-frequency}
% \end{figure}

Specific stances within this AI safety opinion space might be mapped to points along this dimension. For example:
\begin{itemize}
    \item[\textbf{A}:] AI development should proceed unhindered by regulation, focusing on maximizing progress.
    \item[\textbf{B}:] AI safety should primarily be governed by industry self-regulation and market forces.
    \item[\textbf{C}:] Government regulation of AI development is important for ensuring public safety and ethical alignment.
    \item[\textbf{D}:] Development of AI capabilities beyond a certain threshold (e.g., superhuman intelligence) should be permanently banned.
\end{itemize}
Our framework aims to identify an interval (depicted in blue) on this 1D spectrum that best represents the overall consensus, considering both the approval patterns of individuals and the relative salience (frequency) of different positions.

% \ben{Do we clearly justify the "single approval interval per voter" assumption anywhere?}\nimrod{I guess that we can naturally from single-peaked}

\section{Problem Definition}
\label{sec:problem_definition}

We address a learning problem in one dimension (1D). We are provided with $n$ initial intervals, denoted $I_c \subset \mathbb{R}$ for $c=1, \dots, n$. These can be conceptualized as representing the approval region of $n$ distinct voters over a continuous space of issues $x \in \mathbb{R}$. The primary goal is to determine a single hypothesis interval $\hat{I} \subset \mathbb{R}$ that best summarizes or agrees with these voter intervals, according to an objective function defined below. This is considered within a passive learning framework where issues are sampled i.i.d.\ from an underlying, possibly unknown, distribution $P(x)$\footnote{For example, in the case of Pol.is, this underlying distribution is generated by the participants of the deliberation writing statements.}.

\begin{definition}[Individual Label Function]
Let $I_c \subset \mathbb{R}$ be an individual voter interval. The label function $L_c(x)$ for an issue $x \in \mathbb{R}$ with respect to interval $I_c$ is defined as:
$$L_c(x) = 2 \cdot \indicator{x \in I_c} - 1.$$
\end{definition}
Thus, $L_c(x) = +1$ if $x \in I_c$ (the voter approves the issue) and $L_c(x) = -1$ if $x \notin I_c$ (the voter disapproves or does not endorse the issue within this interval).

\begin{definition}[Combined Label Function]
\label{def:combined_label_function}
Let $I_1, \dots, I_n \subset \mathbb{R}$ be $n$ known voter intervals. The combined label function $l(x)$ for any issue $x \in \mathbb{R}$ is the sum of the individual labels:
\begin{align}
&l(x) = \sum_{c=1}^{n} L_c(x) = 2 \sum_{c=1}^{n} \indicator{x \in I_c} - n. 
%\\
%&= \sum_{c=1}^{n} (2 \cdot \indicator{x \in I_c} - 1) = 2 \sum_{c=1}^{n} \indicator{x \in I_c} - n
\end{align}
\end{definition}
So, if an issue $x$ is approved by exactly $k$ of these $n$ voters (i.e., $x$ falls into $k$ of the $I_c$ intervals), then $l(x) = 2k - n$. The function $l(x)$ is an integer taking values in $\{-n, -n+2, \dots, n-2, n\}$. This function $l(x)$ represents the net agreement among the $n$ voters regarding issue $x$; a positive value indicates that more voters approve $x$ than disapprove, and a negative value indicates the opposite. If $\phi(x) = (\sum_{c=1}^n \indicator{x \in I_c})/n$ is the fraction of voters approving $x$, then $l(x) = n(2\phi(x)-1)$.

\subsection{Objective Function: Maximizing Net Agreement} % New subsection for the objective

Our goal is to find a hypothesis interval \(\hat{I} \in \mathcal{H}\), where \(\mathcal{H} = \{[a,b] : a, b \in \mathbb{R}, a \le b\}\) is the class of all closed intervals, that represents the region of maximum consensus. We quantify consensus by the total net agreement accumulated within the chosen interval, weighted by the probability of issues appearing there.

Consider a candidate interval \(\hat{I}\). For each voter \(c\), we want issues within \(\hat{I}\) to align with their preference \(I_c\). That is, we prefer \(x\) to be in both \(\hat{I}\) and \(I_c\) (contributing \(+1\) to \(L_c(x)\)) or outside both, and we dislike \(x\) being in \(\hat{I}\) but outside \(I_c\) (contributing \(-1\) to \(L_c(x)\)). Summing over all voters, the quantity \(l(x) = \sum_c L_c(x)\) reflects the overall net agreement at point \(x\). A desirable consensus interval \(\hat{I}\) should contain issues \(x\) where \(l(x)\) is predominantly positive.
This intuition leads to maximizing the expected value of \(l(x)\) for issues \(x\) that fall within the hypothesis interval \(\hat{I}\).

\begin{definition}[True Objective Function]
\label{def:true_objective}
Given the combined label function \(l(x)\) and the underlying issue distribution \(P(x)\), the quality of a hypothesis interval \(I \in \mathcal{H}\) is measured by the objective function \(\Phi(I)\):
\begin{equation}
    \Phi(I) = \mathbb{E}_{x \sim P} [l(x) \indicator{x \in I}] = \int_{I} l(x) dP(x). 
\end{equation}
The goal is to find an interval $I^* = \arg\max_{I \in \mathcal{H}} \Phi(I)$.
\end{definition}

This objective \(\Phi(I)\) directly calculates the expected contribution to the net agreement score from issues within the interval \(I\). One can also think of maximizing \(\Phi(I)\) as equivalent to maximizing the average agreement between the voters' labels \(L_c(x)\) and the label assigned by the interval \(I\), \(L_I(x) = 2\cdot\indicator{x \in I} - 1\), specifically \(\frac{1}{n}\sum_c \mathbb{E}[L_c(x) L_I(x)]\), up to scaling and additive constants.

In practice, the true distribution \(P(x)\) is unknown. We are instead given a sample of \(m\) issues \(S = \{x_1, \dots, x_m\}\) drawn i.i.d. from \(P(x)\). We therefore use an empirical version of the objective.

\begin{definition}[Empirical Objective Function]
\label{def:empirical_objective}
Given a sample \(S = \{x_1, \dots, x_m\}\) drawn i.i.d. from \(P(x)\), the empirical objective function \(\hat{\Phi}_S(I)\) for an interval \(I \in \mathcal{H}\) is the sample average of the net agreement within the interval:
\begin{equation}
\hat{\Phi}_S(I) = \frac{1}{m} \sum_{i=1}^{m} l(x_i) \indicator{x_i \in I}.
\end{equation}
The Empirical Risk Minimization (ERM) principle seeks the interval 
$\hat{I}_{ERM} = \arg\max_{I \in \mathcal{H}} \hat{\Phi}_S(I)$.
\end{definition}
Observe that maximizing \(\hat{\Phi}_S(I)\) is equivalent to maximizing the unnormalized sum \(\sum_{i=1}^{m} l(x_i) \indicator{x_i \in I}\).

\section{Algorithm for Passive Interval Selection}
\label{sec:algorithm_passive}

The function $l(x)$ is a step function, with its value changing only at the endpoints $a_c, b_c$ of the given voter intervals $I_c$. The problem of maximizing $\hat{\Phi}_S(\hat{I})$ is an Empirical Risk Minimization (ERM) task.

\begin{proposition}
\label{prop:interval_representation_revised}
The interval $\hat{I}_{ERM}$ that maximizes the empirical sum $\sum_{x_i \in \hat{I}} l(x_i)$ over $m$ sample issues $x_1, \dots, x_m$ (assuming $m \ge 1$) can be chosen such that its endpoints are sample issues. That is, $\hat{I}_{ERM} = [x_{(j)}, x_{(k)}]$ for some $1 \le j \le k \le m$, where $x_{(1)} \le x_{(2)} \le \dots \le x_{(m)}$ are the sorted sample issues.
\end{proposition}
\begin{proof}
Let $S(\hat{I}) = \sum_{i=1}^{m} l(x_i) \indicator{x_i \in \hat{I}}$. Consider an interval $I'=[a,b]$ that maximizes this sum. The value of $S(I')$ depends only on the set of sample issues contained within $I'$. If $I'$ contains no sample issues, $S(I')=0$. If $I'$ contains sample issues, let $x_{(j')}$ be the smallest sample issue in $I'$ and $x_{(k')}$ be the largest sample issue in $I'$. The interval $[x_{(j')}, x_{(k')}]$ contains exactly the same set of sample issues as $I'$. Therefore, $S([x_{(j')}, x_{(k')}]) = S(I')$. Thus, an optimal interval whose endpoints are from the set of sample issues can always be constructed.
\end{proof}

Proposition \ref{prop:interval_representation_revised} implies that the problem reduces to finding indices $j^*$ and $k^*$ (with $1 \le j^* \le k^* \le m$) that maximize the sum $S_{j,k} = \sum_{i=j}^{k} h_{(i)}$, where $h_{(i)} = l(x_{(i)})$ is the combined label of the $i$-th sorted sample issue $x_{(i)}$. This is the Maximum Subarray Sum problem on the array of scores $H = [h_{(1)}, \dots, h_{(m)}]$.

\begin{algorithm}[tb]
\caption{ERM for 1D Interval Selection (assuming $m \ge 1$)}
\label{alg:erm_1d_interval_revised}
\begin{algorithmic}[1]
\State \textbf{Input:} Sample issues $x_1, \dots, x_m$; voter intervals $I_1, \dots, I_n$.
\State \textbf{Output:} ERM interval $\hat{I}_{ERM} = [x_{(j^*)}, x_{(k^*)}]$.

\State For each sample issue $x_s$, compute its score $l(x_s) = 2 \sum_{c=1}^{n} \indicator{x_s \in I_c} - n$.
\State Sort the sample issues to get $x_{(1)} \le x_{(2)} \le \dots \le x_{(m)}$. Let $h_{(i)} = l(x_{(i)})$.

\State Initialize $S_{max} \leftarrow -\infty$.
\State Initialize $S_{current} \leftarrow 0$.
\State Initialize $j^* \leftarrow 1$, $k^* \leftarrow 1$ (placeholders, assuming $m \ge 1$).
\State Initialize $j_{cand} \leftarrow 1$.

\For{$i = 1$ to $m$}
    \State $S_{current} \leftarrow S_{current} + h_{(i)}$
    \If{$S_{current} > S_{max}$}
        \State $S_{max} \leftarrow S_{current}$
        \State $j^* \leftarrow j_{cand}$
        \State $k^* \leftarrow i$
    \EndIf
    \If{$S_{current} < 0$}
        \State $S_{current} \leftarrow 0$
        \State $j_{cand} \leftarrow i + 1$
    \EndIf
\EndFor
\If{$S_{max} = -\infty$ and $m \ge 1$} \Comment{All $h_{(i)}$ are negative or $m=0$ case not fully handled by above loop if no positive subarray. If all $h_{(i)}$ are negative, find max $h_{(i)}$.}
    \State $S_{max} \leftarrow h_{(1)}$; $j^* \leftarrow 1$; $k^* \leftarrow 1$;
    \For{$i = 2$ to $m$}
        \If{$h_{(i)} > S_{max}$}
            \State $S_{max} \leftarrow h_{(i)}$; $j^* \leftarrow i$; $k^* \leftarrow i$;
        \EndIf
    \EndFor
\EndIf

\State $\hat{I}_{ERM} = [x_{(j^*)}, x_{(k^*)}]$.
\State \Return $\hat{I}_{ERM}$.
\end{algorithmic}
\end{algorithm}
Algorithm \ref{alg:erm_1d_interval_revised} details this process. The Maximum Subarray Sum component (lines 5-15) uses Kadane's algorithm \cite{bentley1984programming}. The standard Kadane's algorithm finds the maximum sum; if all elements are negative, it might return 0 (for an empty subarray) or the largest single negative element depending on initialization. The version presented, initializing $S_{max}$ to $-\infty$, is designed to find the interval $[x_{(j)},x_{(k)}]$ maximizing $\sum_{i=j}^k l(x_{(i)})$, even if this sum is negative. If all $l(x_i)$ are negative, it will find the $x_{(i)}$ with the least negative $l(x_{(i)})$. A slight modification (lines 16-21) ensures that if all subarray sums are negative (including single elements), it correctly picks the one with the largest (least negative) sum, corresponding to a single point interval.

\subsection{Algorithmic Complexity of ERM}
\label{prop:algo_complexity_final}
The ERM interval $\hat{I}_{ERM}$ can be found efficiently. The primary computational steps are:
\begin{enumerate}
    \item Computing the scores $l(x_i)$ for all $m$ sample points. Naively, this takes $O(nm)$ time by checking each point against each of the $n$ voter intervals $I_c$.
    \item Sorting the sample points $x_i$ to obtain $x_{(i)}$ and their associated scores $h_{(i)}$. This takes $O(m \log m)$ time.
    \item Applying Kadane's algorithm (Algorithm \ref{alg:erm_1d_interval_revised}, lines 5-15) to the $m$ scores $h_{(i)}$, which runs in $O(m)$ time.
\end{enumerate}
The total time complexity using this naive score computation is $O(nm + m \log m)$.
Alternatively, a sweep-line algorithm can be used to compute all $l(x_i)$ values more efficiently. This involves creating a sorted list of all $2n$ endpoints of the given intervals $I_c$ and all $m$ sample points $x_i$. The sweep processes these $O(n+m)$ points in order. As an endpoint of an $I_c$ is crossed, the count of active intervals (used to determine $l(x)$ for points in that segment) is updated. When a sample point $x_i$ is encountered, its $l(x_i)$ value is computed based on the current count of intervals it falls into. This approach takes $O((n+m)\log(n+m))$ time for both sorting the critical points and performing the sweep to assign scores. If the sample points $x_i$ are already sorted by this process, the subsequent Kadane's algorithm step takes $O(m)$. Thus, the total complexity using the sweep-line method is $O((n+m)\log(n+m))$.

\section{Theoretical Analysis}
\label{sec:theoretical_analysis}

We establish Probably Approximately Correct (PAC) learning guarantees for the ERM approach. Let $\mathcal{H}$ be the hypothesis class of all intervals $[a,b] \subset \mathbb{R}$. The true objective for $R \in \mathcal{H}$ is $\Phi(R) = \mathbb{E}_{x \sim P}[l(x)\indicator{x \in R}]$, and the empirical objective is $\hat{\Phi}_S(R) = \frac{1}{m} \sum_{i=1}^m l(x_i)\indicator{x_i \in R}$.
To analyze the uniform convergence of $\hat{\Phi}_S(R)$ to $\Phi(R)$, we consider the properties of the function class $g_R(x) = l(x)\indicator{x \in R}$. The complexity of uniform convergence depends on the function class $\mathcal{G} = \{g_R(x) : R \in \mathcal{H}\}$. The functions $g_R(x) \in \mathcal{G}$ are bounded by $M_0 = \max_x |l(x)| = n$.

\begin{lemma}
\label{lemma:pdim_1d}
The pseudo-dimension of the function class $\mathcal{G} = \{g_R(x) = l(x)\indicator{x \in R} : R \text{ is an interval in } \mathbb{R}\}$, denoted $Pdim(\mathcal{G})$, is 2, provided $l(x)$ is not identically zero on $\mathbb{R}$.
\end{lemma}
\begin{proof}
(Sketch) To show $Pdim(\mathcal{G}) \ge 2$: Choose two points $x_1 < x_2$ where $l(x_1), l(x_2) \neq 0$. Appropriate thresholds $r_1, r_2$ can be chosen such that the outcomes $\mathbf{1}\{l(x_j)\indicator{x_j \in R} > r_j\}$ depend on $\indicator{x_j \in R}$ or $1-\indicator{x_j \in R}$. Since intervals can shatter two points for indicators ($\mathrm{VCdim}=2$), $\mathcal{G}$ can pseudo-shatter two points.
To show $Pdim(\mathcal{G}) \le 2$: For any three points $x_1 < x_2 < x_3$ and any thresholds $r_1,r_2,r_3$, the pattern of $(\indicator{x_1 \in R}, \indicator{x_2 \in R}, \indicator{x_3 \in R})$ cannot achieve $(1,0,1)$. This limitation on indicator patterns restricts the achievable patterns for $(b_1(R),b_2(R),b_3(R))$, preventing pseudo-shattering of 3 points. (Detailed proof in the appendix).
\end{proof}

\begin{theorem}[Sample Complexity]
\label{thm:sample_complexity_formal}
Let $Pdim(\mathcal{G})=d_{PD}=2$ and $M_0 = n$. For any $\epsilon > 0$ and $\delta \in (0,1)$, if the number of i.i.d. samples $m$ satisfies
% \begin{align}
% m \ge \frac{32M_0^2}{\epsilon^2} \left( d_{PD} \ln\left(\frac{32eM_0}{\epsilon}\right) + \ln(4e(d_{PD}+1)) + \ln\left(\frac{1}{\delta}\right) \right)
% \end{align}

\begin{equation}
\begin{split}
m &\ge \frac{32M_0^2}{\epsilon^2} \left( d_{PD} \ln\left(\frac{32eM_0}{\epsilon}\right) \right. \\
& \left. \hspace{2em} + \ln(4e(d_{PD}+1)) + \ln\left(\frac{1}{\delta}\right) \right),
\end{split}
\end{equation}
then with probability at least $1-\delta$, $\sup_{R \in \mathcal{H}} |\Phi(R) - \hat{\Phi}_S(R)| \le \epsilon$.
This implies that the ERM hypothesis $\hat{R}_{ERM}$ satisfies $\Phi(\hat{R}_{ERM}) \ge \Phi(R^*) - 2\epsilon$, where $R^* = \arg\max_{R \in \mathcal{H}} \Phi(R)$. For $d_{PD}=2$,
$$ m \ge \frac{32n^2}{\epsilon^2} \left( 2 \ln\left(\frac{32en}{\epsilon}\right) + \ln(12e) + \ln\left(\frac{1}{\delta}\right) \right). $$
\end{theorem}
\begin{proof} (Sketch)
The proof of the sample complexity (Theorem~\ref{thm:sample_complexity_formal}) relies on standard uniform convergence arguments.
First, the pseudo-dimension of the function class $\mathcal{G} = \{g_R(x) = l(x)\indicator{x \in R}\}$ is determined to be 2, as established in Lemma~\ref{lemma:pdim_1d}.
A uniform convergence bound is then applied. This bound, which is based on methods detailed in Anthony and Bartlett (2009, Ch.~17)~\cite{anthony2009neural}, relates the probability of the maximum deviation between the true expectation $\Phi(R)$ and the empirical expectation $\hat{\Phi}_S(R)$ to the $L_1$-covering numbers of $\mathcal{G}$. Specifically, we use the inequality:
\begin{equation}
\begin{split}
&P\left(\sup_{g \in \mathcal{G}} |E[g] - \hat{E}_S[g]| > \epsilon\right) \\
&\le 4 \mathcal{N}_1(\epsilon/8, \mathcal{G}, 2m) \exp\left(\frac{-m\epsilon^2}{32M_0^2}\right).
\end{split}
\end{equation}
where $M_0 = n$ is the bound on $|l(x)|$. This is derived from the bound given in Theorem 17.1 in Anthony and Bartlett (2009)~\cite{anthony2009neural}.

To bound the $L_1$-covering number $\mathcal{N}_1(\epsilon/8, \mathcal{G}, 2m)$, the functions $g \in \mathcal{G}$ are affinely transformed into a scaled function class $\mathcal{G}' = \{\frac{g+M_0}{2M_0} : g \in \mathcal{G}\}$. Functions in $\mathcal{G}'$ map to $[0,1]$, and this scaling preserves the pseudo-dimension, so $Pdim(\mathcal{G}') = 2$.
Theorem~18.4 from Anthony and Bartlett (2009) is then used to bound the $L_1$-covering number of $\mathcal{G}'$:
\[
\mathcal{N}_1(\eta, \mathcal{G}', N) \le e(d_{PD}+1)\left(\frac{2e}{\eta}\right)^{d_{PD}}.
\]
This bound for $\mathcal{N}_1(\cdot, \mathcal{G}', \cdot)$ is related back to $\mathcal{N}_1(\cdot, \mathcal{G}, \cdot)$ by adjusting the scale factor, yielding:
\[
\mathcal{N}_1(\epsilon/8, \mathcal{G}, 2m) \le e(d_{PD}+1)\left(\frac{32eM_0}{\epsilon}\right)^{d_{PD}}.
\]
Substituting this covering number bound into the uniform convergence inequality, the resulting expression is set to be less than or equal to $\delta$. This inequality is then solved for the sample size $m$, leading to the sample complexity formula presented in Theorem~\ref{thm:sample_complexity_formal}:
\begin{equation}
\begin{split}
m &\ge \frac{32M_0^2}{\epsilon^2} \left( d_{PD} \ln\left(\frac{32eM_0}{\epsilon}\right) + \ln(4e(d_{PD}+1)) \right. \\
&\qquad \left. + \ln\left(\frac{1}{\delta}\right) \right).
\end{split}
\end{equation}
This ensures that with probability at least $1-\delta$, the empirical objective is uniformly within $\epsilon$ of the true objective for all hypotheses in $\mathcal{H}$. A detailed proof is presented in the appendix.
\end{proof}
The above sample complexity is $O\left(\frac{n^2}{\epsilon^2}\left(\ln\frac{n}{\epsilon} + \ln\frac{1}{\delta}\right)\right)$.

% \section{Model/Approach}

% \ben{Start with Passive PAC for Polis setting. Ignore other stuff until later.}

% So it is basically two steps. (1) Show that you have enough labels to approximate the score function reasonably well. Then (2), do the PAC stuff using the approximated score function. 

\section{Experiments}

In this section we describe our experimental framework, introduce specific experiments, and present their results.
The primary goals of our current experiments are:

\begin{enumerate}
    \item Gauge the tightness of the bound on sample complexity identified in \autoref{thm:sample_complexity_formal}.
    \item Explore methods by which we may identify optimal consensus regions using a reduced number of queries.
\end{enumerate}

\subsection{Experimental Framework}

% \ben{We essentially follow Algorithm 1 but some details need clarification for the experimental setting.}

% \ben{Reminder - include final values on things like number of voters, approval interval sizes, level of discretization}

We begin by describing the procedures we follow in initializing an experimental setting. This includes how we select the regions approved by each voter, and the distributions from which we sample. Through all experiments presented in this paper, we have used $n = 100$ voters and performed 100 randomly initialized trials with each set of parameters.

\subsubsection{Constructing Voter Approval Intervals}

% \ben{Note to discuss notation/terminology here with Carter to ensure consistency.}

Each voter approves a single interval in our 1D space. The location and size of this region is not dictated by our theoretical results. To simulate a wide variety of users we construct each voter's approvals based on three parameters: minimum width $w_\text{min}$, maximum width $w_\text{max}$, and center point $p$. These parameters allow us to simulate settings where users are highly agreeable (very large approval intervals), highly disagreeable (very small approval intervals), and where both types exist (large gap between $w_\text{min}$ and $w_\text{max}$).

Interval construction occurs as follows: We select a width $w \in [w_\text{min}, w_\text{max}]$, and a center point $p \in [0, 1]$, both uniformly at random.
From $p$ we construct the interval $(p - \frac{w}{2}, p + \frac{w}{2})$. When this would result in the interval extending above 1, or below 0, we add the ``missing'' width to the other side of the interval (thus, $p$ is not always the actual center point of the interval).

In our experiments, we use four specific ranges of approved region sizes. These are chosen to demonstrate the quality of outcomes across a range of scenarios. The absence of a small region size is intentional -- when all voters approve very small regions, there is typically no area with a positive $l(x)$, resulting in an optimal ``region'' containing a single point.

\subsubsection{Sampling Points}

Through our experiments we consider three distributions from which we sample points. As with voter approval intervals, the choice of these distributions does not affect the applicability of our theoretical results to our experiments but provides insight as to how extensions of our theoretical algorithm could behave under differing empirical conditions. In particular, our sampling distributions are:

\begin{enumerate}
    \item $U(0, 1)$; The uniform distribution between 0 and 1.
    \item $N(0, 1, 0.5, 0.1)$; The truncated Normal distribution with $\mu = 0.5$, $\sigma = 0.1$, bounded in $[0, 1]$.
    \item $Exp(4)$; The truncated Exponential distribution with $\lambda = 4$, bounded in $[0, 1]$.
\end{enumerate}

\subsection{Experimental Results}

Our experimental results are divided into three sections. We first examine how tightness of the lower bound on sample complexity from \Cref{sec:theoretical_analysis}. Subsequently, we explore the quality of two approaches to reducing the total number of queries asked of voters. First by asking fewer voters about each sampled point and, second, by asking each voter about fewer samples.

\subsubsection{Sampling Fewer Points}

We first examine our sample complexity results. In \autoref{fig:varying_num_queries} we decrease the number of sampled points used in identifying the best region beginning from $\frac{1}{\epsilon^2}\left(\ln\frac{n}{\epsilon} + \ln\frac{1}{\delta}\right)$ (a factor of $n^2$ \textit{below} our upper bound) reducing to $10$ sampled points.

\autoref{fig:varying_num_queries} shows that, across distributions, in most cases we can use far fewer samples than our upper bound on sample complexity and attain the same quality as guaranteed theoretically. Of course, as we decrease the number of samples the quality of the best discovered region begins to reduce. In order to develop more empirically practical procedures for identifying consensus regions our subsequent experiments consider two strategies for reducing the total number of sample evaluations performed.

\begin{figure*}
    \centering
    \includegraphics[keepaspectratio,width=\linewidth]{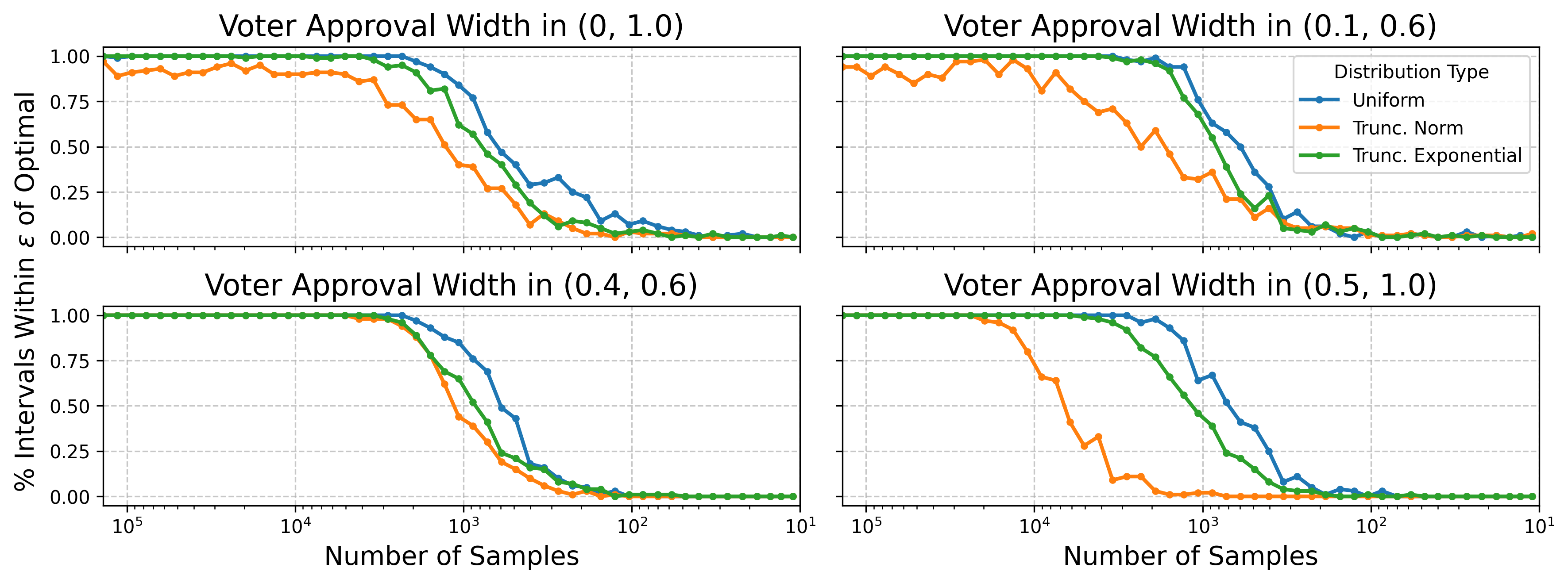}
    \caption{
    The fraction of regions with a score within $\epsilon$ of the optimal region's score as the number of sampled points decreases from $\frac{1}{\epsilon^2}\left(\ln\frac{n}{\epsilon} + \ln\frac{1}{\delta}\right)$ to $10$. 
    Note that, due to computational limitations, the maximum number of samples displayed is a factor of $n^2$ \textit{below} our upper bound found in \Cref{sec:theoretical_analysis}. In general, our approach finds nearly optimal regions using far fewer samples than theoretically necessary. Here $\epsilon = \delta = 0.01$ and we perform 100 trials for each different number of samples to get an empirical estimate of \(\delta\).
    }
    \label{fig:varying_num_queries}
\end{figure*}

\subsubsection{Querying Fewer Voters}

A straightforward approach to reducing the cost of running our procedure is to ask fewer voters whether they approve of each sampled point.
Several techniques exist that are suitable for estimating the number of queries required at each point $x$ to be confident in the value of $l(x)$. However, in this work we simply explore the effect of asking some constant fraction of voters for their evaluation of each sample.

In \autoref{fig:querying_different_num_voters} we decrease the proportion of voters being queried. At each step we sample a random subset of voters with size proportional to the amount shown in the x-axis. This approach provides a linear decrease in the overall cost of identifying a region of consensus. Based on the results shown in \autoref{fig:varying_num_queries} we sample 10,000 points for each experiment shown in \autoref{fig:querying_different_num_voters}.

% \ben{If ambitious I could also include the other figures in an appendix which show by how much the quality of the region decreases.}
However, the results show that this approach very quickly reduces the quality of the underlying region. We hypothesize that this reduction in performance is related to both the relatively low number of voters we use in our experiment ($n = 100$) and the fact that each voter's approved region is sampled uniformly at random. Since the location of one voter's region does not inform the location of any other voter's region, the consensus interval does not generalize well. In contrast, if voter intervals were correlated, we hypothesize that the quality of the interval would reduce at a slower rate with a decreasing fraction of voters sampled.

% skipping even a few voters significantly affects the quality of information that samples provide about the optimal region for all voters.

This hypothesis offers an explanation for why the quality of the optimal region decreases less quickly under the Truncated Normal distribution: As this region is centred around the middle of the 1D interval, when voters approve of larger regions, the optimal region is more likely to be in the center also. This bias towards the center leads to better generalization of the consensus interval.

% \ben{Cite, expand. Chernoff bounds, Wilson score interval, Wald interval, etc. If we make the unfounded assumption that each individual point can be thought of as a Bernoulli.} \carter{Should we get into this in this paper?}
% % https://en.wikipedia.org/wiki/Binomial_proportion_confidence_interval

\begin{figure*}
    \centering
    \includegraphics[keepaspectratio,width=\linewidth]{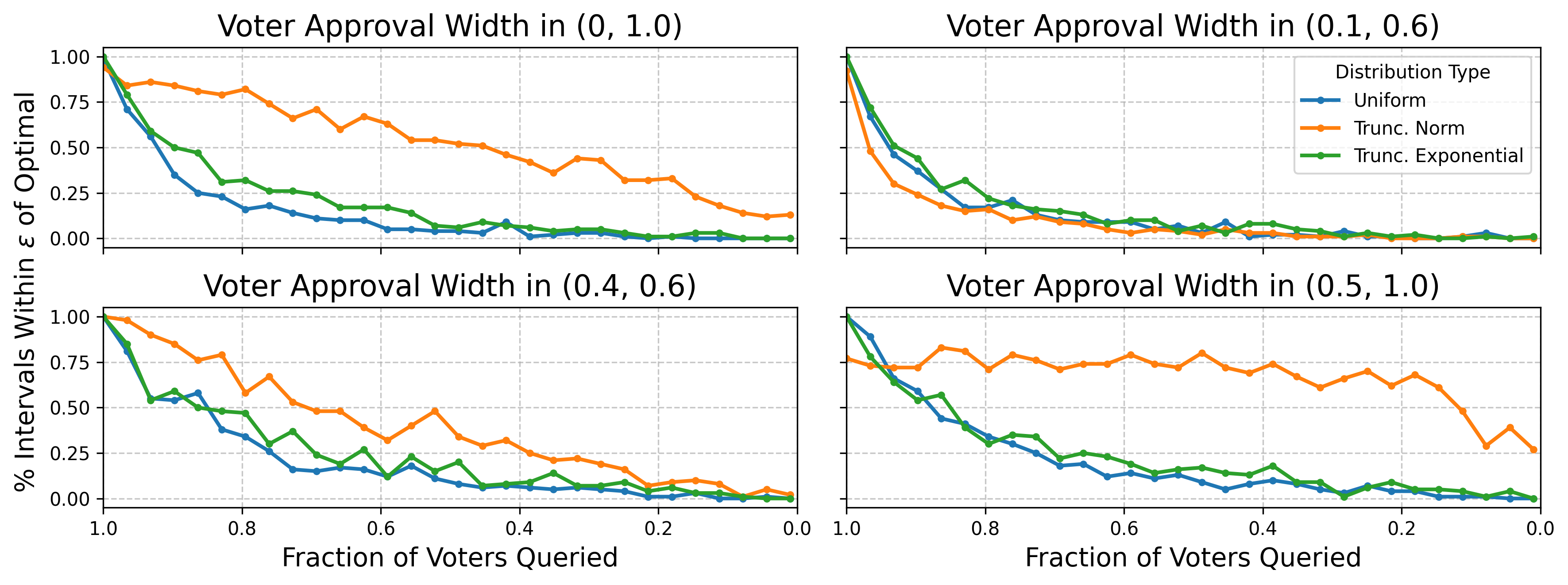}
    \caption{
    The fraction of regions with a score within $\epsilon$ of the optimal region's score as we query a decreasing fraction of voters.
    In all cases $\epsilon = \delta = 0.01$ and we sample 10000 points from the distribution. Each set of parameters is run for 100 randomly initialized trials.
    Quality of the best region found decreases rapidly as fewer voters are queried.
    }
    \label{fig:querying_different_num_voters}
\end{figure*}

% \ben{How does performance change as we decrease the number of voters? Can do this in two ways: (1) uniformly decrease the number of queries per point, (2) query on a point until confident in it's approval score.}

\subsubsection{Querying Voters Intelligently}

Rather than asking a subset of voters about each sampled point, we can instead ask each voter about a subset of sampled points. We can perform a binary search-like procedure to find both endpoints of each voter's approved region.

For each voter $v$, this search proceeds as follows:
\begin{enumerate}
    \item Ask $v$ whether they approve some sampled point until finding a point that they approve. We ask about points at particular fractions through the entire list of ordered samples following the pattern $\{\frac{1}{2}, \frac{1}{4}, \frac{3}{4}, \frac{1}{8}, \frac{2}{8}, ...\}$ (First the median point, then the lower quartile point, then the upper quartile point etc).
    % \item Given some approved sample $p$, query $v$ about the point halfway between $p$ and the leftmost point which may be disapproved (beginning at the left edge). Repeat until all points left of $p$ have a known label.
    \item Given an approved sample point $p$, determine the leftmost extent of $v$'s continuous approval that includes $p$. Let $R_{\text{bound}}$ be the current right boundary of the search (initially $p$) and $L_{\text{bound}}$ be the current left boundary (initially the first sample point in the dataset). Repeatedly query $v$ on the sample point $q$ midway between $L_{\text{bound}}$ and $R_{\text{bound}}$. If $q$ is approved, update $R_{\text{bound}}$ to $q$ (as the approval extends at least this far left). If $q$ is disapproved, update $L_{\text{bound}}$ to be the sample point immediately to the right of $q$. Continue until $L_{\text{bound}}$ and $R_{\text{bound}}$ converge, identifying the left edge of approval and thereby labeling all queried points.
    \item Perform step 2 on the points to the right of $p$.
\end{enumerate}

% Note that when no sampled points overlap with a voter's approved region, we are not able to find the voter's approvals. \ben{This does happen occasionally :/ And needs some explanation :'(} \carter{I would try to do this but I am not sure what to explain}
% \ben{I think just ignoring it is probably fine? I don't think any reviewer is going to read this in so much depth. It's really a relatively minor implementation detail I think}

We show in \autoref{fig:binary_search_cost} the number of sampled points used in the experiment shown in \autoref{fig:varying_num_queries} and the average number of points upon which each voter must be queried in order to find the best region. For example, when we sample $10^5$ points, we typically identify the optimal region using only slightly more than 30 queries per voter. This active approach to querying voters \textit{dramatically} reduces the number of queries required to identify the optimal region. Such an approach is vital for any application of this method to human settings. 

% \ben{Do you think there's any room for this last sentence in the conclusion or elsewhere? It doesn't really fit here but the conclusion doesn't really get into that. It really seems like the binary search method kinda takes this to "practical" rather than "ask a human 100,000 queries."} \carter{I wish we could put this result in the abstract too and intro and conclusion}
% \ben{Yeah, it seems like a pretty strong result actually}

% How would you summarize it? Like "We find that through active learning we can dramatically reduce the number of labels needed on the sampled points to a practical number" or something like that?

% Yea, seems reasonable.

% thinking...

\begin{figure}
    \centering
    \includegraphics[width=\linewidth]{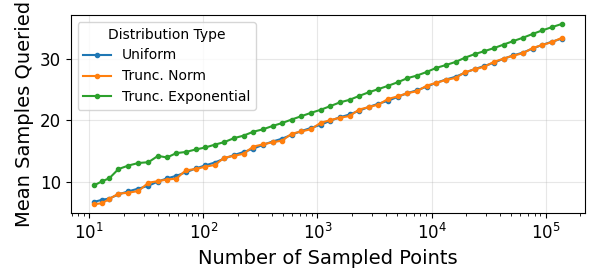}
    \caption{Average number of points required to identify each voter's approved region compared with the total number of sampled points for voters approving a region with width in $0.4, 0.6)$.}
    \label{fig:binary_search_cost}
\end{figure}

\section{Related Work}
\label{sec:related_work}

The problem of identifying consensus or aggregating preferences connects to several research streams.
Qiu \cite{qiu2024representative} presents a ``Representative Social Choice'' model, framing social choice as statistical learning and using VC dimension and Rademacher complexity for generalization bounds. This aligns with our PAC learning approach (using pseudo-dimension). However, our work focuses on finding an optimal 1D consensus interval that maximizes a net agreement score, given known voter intervals and sampled issues, rather than Qiu's axiomatic focus.

Platforms like Polis \cite{small2021polis} use computational techniques like PCA and k-means to identify consensus from user votes. However, its methods are heuristic and descriptive, lacking the formal objective function optimization and provable PAC guarantees central to our work. Our approach formalizes consensus as maximizing $\mathbb{E}[l(x)\indicator{x \in I}]$ and provides an algorithm with theoretical performance bounds.

Halpern et al. 
\cite{halpern2023representation} also addresses challenges in Polis-like systems, specifically incomplete approval votes, by modelling the task as approval-based committee selection. They develop adaptive algorithms to achieve fairness guarantees like Justified Representation. While they also learn from partial information, their goal of committee selection contrasts with our PAC-based optimization of a consensus interval from known voter intervals and sampled issues.

Anshelevich et al. \cite{anshelevich2018approximating} also consider preferences within a metric space, where both voters and alternatives are points, and preferences are determined by proximity. While their use of a metric space for preferences is similar to our 1D issue space, their focus on distortion with ordinal inputs contrasts with our PAC learning approach.

Finally, very related to our work is \cite{elkind2024united}, which studied the problem of consensus formation in metric spaces where coalitions form iteratively around points eliciting wider support. That work, however, abstracts away from the problem of identifying such consensus points, which has been the explicit focus of the work presented here.

\section{Conclusion and Future Work}
\label{sec:conclusion}
This paper introduced a formal framework for passive PAC learning of a consensus interval in a one-dimensional opinion space. We defined an objective function based on aggregated voter approval intervals, presented an efficient ERM algorithm using Kadane's method, and derived sample complexity bounds based on a pseudo-dimension of 2 for the associated function class. These results provide a theoretical basis for identifying regions of maximal approval with provable guarantees.

Future work could pursue several directions:

\textbf{Active Learning}: Developing algorithms for the ``PAC Queries'' setting, where voter labels are acquired strategically to minimize query costs while learning the consensus interval or individual voter intervals.

\textbf{Higher Dimensions}: Extending the model and analysis to find consensus regions in multi-dimensional spaces, which is more directly applicable to complex data like text embeddings. This will involve different hypothesis classes (e.g., axis-aligned rectangles or balls).

\textbf{Alternative Objectives}: Exploring different normative definitions of ``consensus.'' For example, we considered the overlap between voter intervals and the consensus interval as desirable and the overlap between the consensus interval and the complement of voter intervals as a negative. However, one could also argue that a \emph{lack} of overlap between voter intervals and the consensus interval is also undesirable. 
    % And further, a lack of overlap between the consensus statement and the complement of voter intervals could be seen as desirable. 

\textbf{AI Alignment}: Integrating the consensus-finding framework to distill salient regions of acceptable AI expression. Learned consensus intervals could operate as filters or inform the reward mechanisms within RLAIF \cite{lee2023rlaif}.

\textbf{Empirical Validation}: We motivated our paper with applications like online deliberation platforms, but tested our method on synthetic data. However, extending to real-world datasets from platforms like Pol.is is necessary in order to understand how well our algorithms would perform in practice.

The overall goal of this line of research is to identify computationally viable methods for salient consensus elicitation in complex collective decision-making settings.

\bibliographystyle{named}
\bibliography{ijcai25}

\clearpage
\appendix

\section{Proof of Lemma \ref{lemma:pdim_1d} }
\begin{proof}
The pseudo-dimension $Pdim(\mathcal{G})$ is the maximum integer $d$ such that there exists a set of $d$ points $\{x_1, \dots, x_d\}$ that can be pseudo-shattered by $\mathcal{G}$. To pseudo-shatter a set of points means there exist thresholds $r_1, \dots, r_d \in \mathbb{R}$ such that for any binary vector $(b_1, \dots, b_d) \in \{0,1\}^d$, there is a function $g_R \in \mathcal{G}$ (i.e., an interval $R$) for which $\mathbf{1}\{g_R(x_i) > r_i\} = b_i$ for all $i=1, \dots, d$.
If $l(x)$ is identically zero, then $g_R(x)=0$ for all $R, x$. Then $\mathbf{1}\{g_R(x_i)>r_i\} = \mathbf{1}\{0>r_i\}$, which does not depend on $R$. Thus, no points can be shattered, and $Pdim(\mathcal{G})=0$. We now assume $l(x)$ is not identically zero, and thus that some voter has a non-empty approval interval.

\textbf{Part 1:} Show $Pdim(\mathcal{G}) \ge 2$.
We need to show that there exists a set of 2 points that $\mathcal{G}$ can pseudo-shatter. Since $l(x)$ is not identically zero, there exist at least two distinct points $x_1 < x_2$ where $l(x_1) \neq 0$ and $l(x_2) \neq 0$. (If $l(x)$ is non-zero at only one point, $Pdim(\mathcal{G})$ would be 1. The class of intervals allows isolating individual points or pairs, making the existence of two such distinct points $x_1, x_2$ a typical scenario for non-trivial $l(x)$ relevant to this problem).

Let $x_1 < x_2$ be two points such that $l(x_1) \neq 0$ and $l(x_2) \neq 0$. We define thresholds $r_1, r_2$ based on the signs of $l(x_1), l(x_2)$.

Case A: $l(x_1)>0$ and $l(x_2)>0$. Let $r_j = l(x_j)/2$. Then $r_j>0$.
The condition $\mathbf{1}\{l(x_j)\indicator{x_j \in R} > r_j\} = b_j$ becomes $\mathbf{1}\{l(x_j)\indicator{x_j \in R} > l(x_j)/2\} = b_j$.
If $x_j \in R$: $\mathbf{1}\{l(x_j) > l(x_j)/2\} = 1$ (since $l(x_j)>0$).
If $x_j \notin R$: $\mathbf{1}\{0 > l(x_j)/2\} = 0$ (since $l(x_j)/2>0$).
So, $b_j = \indicator{x_j \in R}$.

Case B: $l(x_1)<0$ and $l(x_2)<0$. Let $r_j = l(x_j)/2$. Then $r_j<0$.
The condition $\mathbf{1}\{l(x_j)\indicator{x_j \in R} > r_j\} = b_j$ becomes $\mathbf{1}\{l(x_j)\indicator{x_j \in R} > l(x_j)/2\} = b_j$.
If $x_j \in R$: $\mathbf{1}\{l(x_j) > l(x_j)/2\} = 0$ (since $l(x_j)<0$, e.g., $-2 > -1$ is false).
If $x_j \notin R$: $\mathbf{1}\{0 > l(x_j)/2\} = 1$ (since $l(x_j)/2<0$).
So, $b_j = \indicator{x_j \notin R} = 1 - \indicator{x_j \in R}$.

Case C: $l(x_1)>0$ and $l(x_2)<0$. Let $r_1=l(x_1)/2$ and $r_2=l(x_2)/2$.
Then $b_1 = \indicator{x_1 \in R}$ and $b_2 = \indicator{x_2 \notin R}$.

In all cases (A, B, C, and C with signs swapped), the pair $(b_1,b_2)$ is determined by $(\indicator{x_1 \in R}, \indicator{x_2 \in R})$ or simple transformations thereof. The class of indicator functions $\{\indicator{x \in R} : R \text{ is an interval}\}$ has VC-dimension 2, meaning it can shatter a set of two points $\{x_1, x_2\}$. That is, all 4 patterns for $(\indicator{x_1 \in R}, \indicator{x_2 \in R})$ are achievable by varying $R$:
\begin{itemize}
    \item $(\indicator{x_1 \in R}, \indicator{x_2 \in R}) = (1,1)$ by $R=[x_1,x_2]$.
    \item $(\indicator{x_1 \in R}, \indicator{x_2 \in R}) = (0,0)$ by $R=\emptyset$. % Or an interval not containing x1 or x2
    \item $(\indicator{x_1 \in R}, \indicator{x_2 \in R}) = (1,0)$ by $R=[x_1,x_1]$.
    \item $(\indicator{x_1 \in R}, \indicator{x_2 \in R}) = (0,1)$ by $R=[x_2,x_2]$.
\end{itemize}
In Case A, $(b_1,b_2)$ directly achieves these 4 patterns.
In Case B, if $(\indicator{x_1 \in R}, \indicator{x_2 \in R})$ takes values $(1,1), (0,0), (1,0), (0,1)$, then $(b_1,b_2)$ takes values $(0,0), (1,1), (0,1), (1,0)$ respectively. All 4 patterns are achieved.
In Case C, if $(\indicator{x_1 \in R}, \indicator{x_2 \in R})$ takes values $(1,1), (0,0), (1,0), (0,1)$, then $(b_1,b_2) = (\indicator{x_1 \in R}, 1-\indicator{x_2 \in R})$ takes values $(1,0), (0,1), (1,1), (0,0)$ respectively. All 4 patterns are achieved.
Thus, under the condition that $l(x)$ is non-zero at two distinct points $x_1, x_2$, we can find thresholds $r_1,r_2$ such that all 4 binary patterns for $(b_1,b_2)$ are generated. Therefore, $Pdim(\mathcal{G}) \ge 2$.

\textbf{Part 2:} Show $Pdim(\mathcal{G}) \le 2$.

We must show that no set of 3 points can be pseudo-shattered by $\mathcal{G}$.
Let $x_1 < x_2 < x_3$ be any three distinct points in $\mathbb{R}$. Let $r_1, r_2, r_3$ be any three real-valued thresholds.
Let $b_j(R) = \mathbf{1}\{g_R(x_j) > r_j\} = \mathbf{1}\{l(x_j)\indicator{x_j \in R} > r_j\}$.
For each point $x_j$, $l(x_j)$ is a fixed real number and $r_j$ is a fixed threshold. We define two values for each point $x_j$:
$s_j^{(1)} = \mathbf{1}\{l(x_j) > r_j\}$ (this is the outcome $b_j(R)$ if $x_j \in R$)
$s_j^{(0)} = \mathbf{1}\{0 > r_j\}$ (this is the outcome $b_j(R)$ if $x_j \notin R$)
Thus, for any interval $R$, the generated binary vector $(b_1(R), b_2(R), b_3(R))$ is formed as $(s_1^{(\indicator{x_1 \in R})}, s_2^{(\indicator{x_2 \in R})}, s_3^{(\indicator{x_3 \in R})})$.

Let $f_R(x_j) = \indicator{x_j \in R}$. The vector of indicators $(f_R(x_1), f_R(x_2), f_R(x_3))$ represents which of the three points are included in the interval $R$.
It is a standard result that the class of 1D intervals (a Vapnik-Chervonenkis class) has VC-dimension 2. This means it cannot shatter any set of 3 points. Specifically, for $x_1 < x_2 < x_3$, the indicator pattern $(f_R(x_1)=1, f_R(x_2)=0, f_R(x_3)=1)$ cannot be generated by any interval $R$. If $x_1 \in R$ and $x_3 \in R$ for an interval $R$, then all points between $x_1$ and $x_3$ must also be in $R$, meaning $x_2 \in R$. Thus, $f_R(x_2)$ must be 1 if $f_R(x_1)=1$ and $f_R(x_3)=1$.

Let $P_f = \{ (f_R(x_1), f_R(x_2), f_R(x_3)) : R \text{ is an interval} \}$ be the set of all possible indicator patterns generated on $x_1, x_2, x_3$ by intervals. Since the pattern $(1,0,1)$ is not in $P_f$, the size of $P_f$ is $|P_f| \le 2^3 - 1 = 7$.
The set of output vectors $V = \{ (b_1(R), b_2(R), b_3(R)) : R \text{ is an interval} \}$ is obtained by applying the fixed mapping (defined by $s_j^{(0)}, s_j^{(1)}$ for each $j$) to each vector in $P_f$. That is, $V = \{ (s_1^{(v_1)}, s_2^{(v_2)}, s_3^{(v_3)}) : (v_1,v_2,v_3) \in P_f \}$.
Since the mapping is from $P_f$ to $V$, it must be that $|V| \le |P_f|$.
Therefore, $|V| \le 7$. As $V$ contains at most 7 distinct vectors, it cannot be equal to the set of all $2^3=8$ binary vectors $\{0,1\}^3$.
This means that for any set of 3 points $x_1 < x_2 < x_3$ and any set of thresholds $r_1, r_2, r_3$, there is at least one binary pattern in $\{0,1\}^3$ that cannot be generated by $\mathcal{G}$.
Thus, no set of 3 points can be pseudo-shattered by $\mathcal{G}$, which implies $Pdim(\mathcal{G}) < 3$, and therefore $Pdim(\mathcal{G}) \le 2$.

\textbf{Conclusion:}
Since $Pdim(\mathcal{G}) \ge 2$ and $Pdim(\mathcal{G}) \le 2$, we conclude that $Pdim(\mathcal{G}) = 2$, provided $l(x)$ is not identically zero (specifically, non-zero on at least two points where the interval class can distinguish them as required for the $Pdim \ge 2$ construction).
\end{proof}

\section{Proof of Theorem \ref{thm:sample_complexity_formal}}
\begin{proof}

To derive the sample complexity bound, we first note that the functions $g_R(x) \in \mathcal{G}$ are bounded by $M_0 = \max_x |l(x)| = n$, since $l(x)$ takes integer values in $\{-n, \dots, n\}$. The pseudo-dimension of $\mathcal{G}$, $d_{PD} = Pdim(\mathcal{G})$, is 2 as established in Lemma \ref{lemma:pdim_1d}.

For a class of functions $\mathcal{G}$ whose members $g$ are bounded such that $|g(x)| \le M_0$, the probability that the empirical mean $\hat{E}_S[g]$ deviates from the true mean $E[g]$ by more than $\epsilon$ for any $g \in \mathcal{G}$ is bounded as:
\begin{align*}
 &P\left(\sup_{g \in \mathcal{G}} |E[g] - \hat{E}_S[g]| > \epsilon\right) \\ &\le 4 \mathcal{N}_1(\epsilon/8, \mathcal{G}, 2m) \exp\left(\frac{-m\epsilon^2}{32M_0^2}\right)   
\end{align*}
 
Here, $\mathcal{N}_1(\eta, \mathcal{G}, 2m)$ is the $L_1$-covering number of the class $\mathcal{G}$ at scale $\eta$ for $2m$ points.

To utilize pseudo-dimension bounds for covering numbers, which are typically stated for functions mapping to $[0,1]$, we introduce a scaled function class $\mathcal{G}'$. Let $g \in \mathcal{G}$. Define $g' = \frac{g+M_0}{2M_0}$. These functions $g' \in \mathcal{G}'$ map to $[0,1]$. The pseudo-dimension of $\mathcal{G}'$ is the same as that of $\mathcal{G}$, i.e., $Pdim(\mathcal{G}') = d_{PD} = 2$, because pseudo-dimension is invariant under positive affine transformations of the function outputs.
Anthony and Bartlett (2009, Theorem 18.4) provide a bound for the $L_1$-covering number of a class of functions mapping to $[0,1]$ in terms of its pseudo-dimension $d_{PD}$:
$$ \mathcal{N}_1(\eta, \mathcal{G}', N) \le e(d_{PD}+1)\left(\frac{2e}{\eta}\right)^{d_{PD}} $$
for $N$ points.

We need to relate $\mathcal{N}_1(\epsilon/8, \mathcal{G}, 2m)$ to a covering number for $\mathcal{G}'$. The $L_1$-distance between two functions $g_1, g_2 \in \mathcal{G}$ relates to their scaled versions $g_1', g_2' \in \mathcal{G}'$ by $E_S[|g_1 - g_2|] = E_S[| (2M_0 g_1' - M_0) - (2M_0 g_2' - M_0) |] = 2M_0 E_S[|g_1' - g_2'|]$.
Thus, an $\eta'$-cover for $\mathcal{G}'$ corresponds to a $2M_0\eta'$-cover for $\mathcal{G}$. This implies $\mathcal{N}_1(\epsilon_0, \mathcal{G}, 2m) = \mathcal{N}_1(\epsilon_0/(2M_0), \mathcal{G}', 2m)$.
For $\epsilon_0 = \epsilon/8$, we have:
\begin{align*}
\mathcal{N}_1(\epsilon/8, \mathcal{G}, 2m) &= \mathcal{N}_1\left(\frac{\epsilon/8}{2M_0}, \mathcal{G}', 2m\right)\\ &= \mathcal{N}_1\left(\frac{\epsilon}{16M_0}, \mathcal{G}', 2m\right)
\end{align*}

Let $\eta' = \frac{\epsilon}{16M_0}$. Using the covering number bound for $\mathcal{G}'$ with $N=2m$ points:
 
\begin{align*}
\mathcal{N}_1\left(\frac{\epsilon}{16M_0}, \mathcal{G}', 2m\right)  &\le e(d_{PD}+1)\left(\frac{2e}{\epsilon/(16M_0)}\right)^{d_{PD}}\\ 
&= e(d_{PD}+1)\left(\frac{32eM_0}{\epsilon}\right)^{d_{PD}}
\end{align*}

Substituting this into the uniform convergence probability bound:
\begin{align*}
&P\left(\sup_{g \in \mathcal{G}} |E[g] - \hat{E}_S[g]| > \epsilon\right) \\
&\le 4 e(d_{PD}+1)\left(\frac{32eM_0}{\epsilon}\right)^{d_{PD}} \exp\left(\frac{-m\epsilon^2}{32M_0^2}\right)    
\end{align*}

We want this probability to be at most $\delta$. So we set:
$$ 4 e(d_{PD}+1)\left(\frac{32eM_0}{\epsilon}\right)^{d_{PD}} \exp\left(\frac{-m\epsilon^2}{32M_0^2}\right) \le \delta $$
To solve for $m$, we can rearrange and take the natural logarithm:
$$ \exp\left(\frac{m\epsilon^2}{32M_0^2}\right) \ge \frac{4 e(d_{PD}+1)\left(\frac{32eM_0}{\epsilon}\right)^{d_{PD}}}{\delta} $$
$$ \frac{m\epsilon^2}{32M_0^2} \ge \ln\left(4 e(d_{PD}+1)\left(\frac{32eM_0}{\epsilon}\right)^{d_{PD}} \cdot \frac{1}{\delta}\right) $$
$$ \frac{m\epsilon^2}{32M_0^2} \ge \ln(4e(d_{PD}+1)) + d_{PD}\ln\left(\frac{32eM_0}{\epsilon}\right) + \ln\left(\frac{1}{\delta}\right) $$
Finally, solving for $m$:
$$ m \ge \frac{32M_0^2}{\epsilon^2} \left( d_{PD} \ln\left(\frac{32eM_0}{\epsilon}\right) + \ln(4e(d_{PD}+1)) + \ln\left(\frac{1}{\delta}\right) \right) $$.
\end{proof}

\end{document}